%% file: main.tex
\documentclass{article}
\usepackage{graphicx} 
\usepackage{amsmath}
\usepackage{amssymb}
\usepackage{amsthm}
\usepackage[margin=1.5in]{geometry}
\usepackage{appendix}
\usepackage{scalerel,stackengine}
\usepackage{color}
\usepackage{transparent}
\usepackage{subcaption}
\usepackage{relsize}
\usepackage{cancel}
\usepackage{url}

\title{How Regularization Terms Make Invertible Neural Networks Bayesian Point Estimators\footnote{Preprint. Submitted to Inverse Problems (IOP Publishing), October 2025.}}
\author{Nick Heilenkötter\thanks{Center for Industrial Mathematics, University of Bremen, Germany, \texttt{heilenkoetter@uni-bremen.de}}}
\date{}
\newcommand\equalhat{\mathrel{\stackon[1.5pt]{=}{\stretchto{%
    \scalerel*[\widthof{=}]{\wedge}{\rule{1ex}{3ex}}}{0.5ex}}}}

\newtheorem{remark}{Remark}
\newtheorem{theorem}{Theorem}
\newtheorem{lemma}{Lemma}
\newtheorem{corollary}{Corollary}

\newcommand{\X}{X}
\newcommand{\Y}{Y}
\newcommand{\pnoise}{p_H}

\newcommand{\prior}{p_X}

\newcommand{\py}{p_Y}

\newcommand{\model}{%
  \mathchoice
    {{\varphi^{\!\resizebox{.5em}{.18em}{$\to$}}}} 
    {{\varphi^{\hspace{0.02em}\resizebox{.5em}{.18em}{$\to$}}}}  
    {{\varphi^{\!\resizebox{.5em}{.18em}{$\to$}}}}  
    {{\varphi^{\!\resizebox{.5em}{.18em}{$\to$}}}} 
}
\newcommand{\invmodel}{%
  \mathchoice
    {{\varphi^{\hspace{-0.05em}\resizebox{.5em}{.18em}{$\gets$}}}} 
    {{\varphi^{\resizebox{.5em}{.18em}{$\gets$}}}}  
    {{\varphi^{\hspace{-0.05em}\resizebox{.5em}{.18em}{$\gets$}}}}  
    {{\varphi^{\hspace{-0.05em}\resizebox{.5em}{.18em}{$\gets$}}}} 
}
\newcommand{\A}{A}
\newcommand{\Aadj}{\A^\ast}

\newcommand{\x}{x}
\newcommand{\y}{y}

\newcommand{\zdelta}{z^\delta}

\newcommand{\ydelta}{y^\delta}
\newcommand{\inv}{^{-1}}

\newcommand{\xmap}{\x_{\mathrm{MAP}}}
\newcommand{\xpm}{\x_{\mathrm{PM}}}
\newcommand{\xlogdet}{\x_{\mathrm{logdet}}}
\newcommand{\modelapprox}{\model_\mathrm{\!\!\!approx}}
\newcommand{\modelreco}{\model_\mathrm{\!\!\!reco}}
\newcommand{\modeldiv}{\model_\mathrm{\!\!\!div}}
\newcommand{\modelogdet}{\model_\mathrm{\!\!\!logdet}}

\newcommand{\invmodelogdet}{\invmodel_\mathrm{\!\!\!logdet}}
\newcommand{\deltachosen}{\hat{\delta}}
\DeclareMathOperator{\supp}{supp}
\DeclareMathOperator{\Lip}{Lip}
\DeclareMathOperator{\Diffeo}{Diff}

\newcommand{\AadjA}{\Aadj\!\A}
\newcommand{\Id}{\mathrm{Id}}
\newcommand{\priorpushed}{\model_{\mathrm{\!\!\!logdet}\#}\prior}

\newcommand{\diff}{\mathrm{d}}
\newcommand{\Diff}{\mathrm{D}}
\newcommand{\tr}{\mathrm{tr}}
\newcommand{\expect}{\mathbb{E}}
\newcommand{\norm}[1]{\left\lVert#1\right\rVert}
\newcommand{\R}{\mathbb{R}}
\begin{document}

\maketitle
\textbf{Abstract:}
Can regularization terms in the training of invertible neural networks lead to known Bayesian point estimators in reconstruction? Invertible networks are attractive for inverse problems due to their inherent stability and interpretability. Recently, optimization strategies for invertible neural networks that approximate either a reconstruction map or the forward operator have been studied from a Bayesian perspective, but each has limitations.
To address this, we introduce and analyze two regularization terms for the network training that, upon inversion of the network, recover properties of classical Bayesian point estimators: while the first can be connected to the posterior mean, the second resembles the MAP estimator.
Our theoretical analysis characterizes how each loss shapes both the learned forward operator and its inverse reconstruction map. Numerical experiments support our findings and demonstrate how these loss‑term regularizers introduce data-dependence in a stable and interpretable way.

\section*{Introduction}
Whenever a quantity of interest cannot be observed directly but only through an indirect measurement process or in the presence of noise, one is faced with an inverse problem. To stabilize the reconstruction and mitigate the information loss inherent in the measurement, it is necessary to incorporate additional knowledge about the unknown data — its prior distribution, which encodes what one expects the reconstruction to resemble, such as the characteristic features of natural images. Yet our ability to describe natural images in an explicit, algorithmic form remains quite limited. Fortunately, recent years have seen the emergence of data-driven approaches that enable the construction of priors directly from collections of representative samples.

\begin{figure}[ht]
    \centering
    \def\svgwidth{\linewidth}
    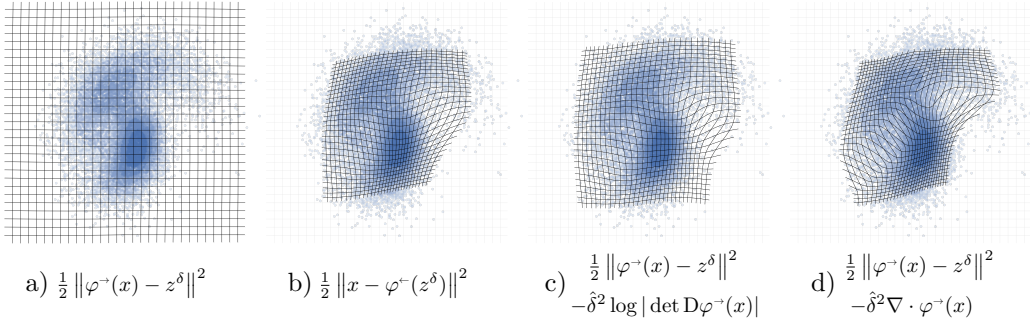
    \caption{How the reconstruction methods resulting from the studied training strategies map an equidistant grid (gray) when $\A = \Id$: 
    a) approximation training induces no regularization; 
    b) reconstruction training approximates the posterior mean; 
    c) log‑determinant regularization links to a smoothed posterior mean; 
    d) divergence‑based regularization approximates the MAP estimator, visibly pulling toward the peaks of the prior.}
    \label{fig:first_page_grids}
\end{figure}

While these approaches often surpass classical methods in reconstruction quality, many of them lack theoretical guarantees and remain difficult to interpret. A promising direction explored recently \cite{Arndt_2023, Arndt_2024, Arndt_2025, Izmailov_2020} involves invertible neural networks. Thanks to their bidirectional structure, a single network can simultaneously approximate the forward operator and serve as a reconstruction method, with stability ensured by the architecture itself. This hybrid use makes it possible to assess deviations from a known forward operator – or even replace it by a data-based version – while maintaining interpretability of the reconstruction process by the learned measurement model and vice versa. This dual capability is particularly relevant in applications where both high-fidelity reconstructions and a faithful representation of the measurement process are critical, such as scientific imaging and medical diagnostics. It can also be of interest when the evaluation of the forward operator is computationally expensive, which is common in PDE-based inverse problems.

Yet a fundamental question remains: Which training strategies enable effective learning in both directions simultaneously? Prior work \cite{Arndt_2023, Arndt_2024} has analyzed training strategies that approximate the forward operator to obtain convergent regularizations, yielding strong performance in low-noise settings. However, such approaches lack data-dependent regularization, which induces prior knowledge and becomes essential in high-noise regimes. In contrast, directly training the reconstruction introduces data-dependent regularization by approximating the posterior mean, but often performs worse in operator approximation, i.e. data consistency, and can lead to ``regression to the mean'', manifesting in lower visual fidelity \cite{Delbracio_2023}.

Motivated by these limitations, we seek alternative training strategies that combine forward operator approximation and good reconstruction quality within a single invertible network. In this paper, we present a theoretical study of two regularization terms that embed data‑dependent priors into the training objective while preserving forward operator approximation. To the best of our knowledge, explicitly using such regularization penalties to impart Bayesian‑style data dependence in invertible network training is novel. Although our focus is theoretical, these concepts can be applied in a large range of practical challenges in inverse problems; we discuss possible use cases in the conclusion.
Throughout our analyses, we focus on an intuition-driven perspective and aim to provide insights into the concepts behind different Bayesian point estimators. We begin our search with an approach motivated by the structure of Normalizing Flows and analyze the resulting reconstruction method, exposing a close connection to the posterior mean point estimate. Finally, for a broader class of problems, we introduce a modified regularization term and show that, under certain conditions, it yields the well-known MAP estimate, resulting in the desired combination of accurate forward operator approximation and high-quality reconstruction. In summary, this work proposes a new class of loss functions that reconcile these competing objectives and provides rigorous analysis to support their understanding and practical deployment in inverse problems.

\subsubsection*{Paper Structure and Key Contributions}

In our effort to introduce regularization terms in neural network training for inverse problems and understand their behavior, we first provide a concise summary of the problem setup, which also defines the assumptions underlying the subsequent theorems. We then present theoretical preliminaries, including a discussion of Bayesian point estimators and their interpretation as score-based regularizations, and we prove an extension of Tweedie’s formula to linear inverse problems. The main part of this work presents and analyzes two regularization terms for the training of the model $\model$ from a theoretical perspective: for $\log|\det\Diff\model|$, which is motivated by Normalizing Flow approaches, we establish a connection to posterior mean estimates; for the regularization with $\nabla\cdot\model$, we derive a result demonstrating a close relation to MAP estimation. We conclude by validating these findings through numerical toy examples and discussing their practical implications for reconstruction with invertible neural networks.

\section{Problem Setting and Assumptions}
We consider linear inverse problems with additive Gaussian noise, i.e.,
\begin{align}\label{problem:orig}
    \ydelta = \A\x+\eta, \quad \eta\sim\pnoise\equalhat\mathcal{N}(0,\delta^2)
\end{align}
for $\x \in \X$, where $\A : \X \to \Y$ may be ill-conditioned. Since neural networks operate on finite-dimensional spaces in practice, we restrict our analysis to the case where $\X$ and $\Y$ are finite-dimensional vector spaces, and identify $\X\equalhat \R^{n_\X}$ and $\Y\equalhat\R^{n_\Y}$.

In the Bayesian framework, we assume that $\x$ and the noise $\eta$ are independent, and that $\x$ is distributed according to a prior with a continuously differentiable density $\prior: \X \to \mathbb{R}_{>0}$.\footnote{To shorten the notation, we assume that the prior has global support. Alternatively, all results could be restricted to $\mathrm{supp}\,\prior$, which is the practically relevant subset of $\X$.} In practice, this prior may encode the structure of natural images or prior knowledge about physical parameters. While the prior is typically not known in closed form, we assume it is implicitly accessible through training data. We further assume that $\prior$ decays sufficiently fast to ensure the existence of all expectations involved in the loss functions considered below.

If invertible network architectures are used, the input and output spaces must coincide. In particular, when $\Y=\X$, those architectures can be applied directly to approximate either forward or reconstruction operators. Otherwise, we follow the approach of \cite{Arndt_2023,Arndt_2024}, assume knowledge of the adjoint operator and consider the normal equation
\begin{align}\label{problem:normal}
    \zdelta = \Aadj \ydelta = \AadjA\x+\Aadj \eta,
\end{align}
where $\zdelta\in\X$. This can be understood as a new inverse problem with forward operator $\AadjA$ and noise $\Aadj\eta$. Note that while $\AadjA$ is typically more ill-conditioned than the original formulation, we have $\ker(\AadjA)=\ker(\A)$, so no information is lost.

The goal of this work is to optimize such bijective networks $\model:\X\to\X$ that approximate the forward mapping $\x\mapsto \zdelta$ (or $\x\mapsto\ydelta$, where stated explicitly), so that the inverse network can subsequently be used for reconstruction:
\begin{align*}
    \x^\ast = \invmodel(\zdelta),\quad \text{ where } \invmodel:=(\model)\inv.
\end{align*}
In our theoretical analysis, we optimize models $\model \in \Diffeo^2(\X)\subset C^2(\X,\X)$, where $\Diffeo^2(\X)$ contains all bijective $C^2$-functions with $C^2$-inverse. 
The models are trained on datasets of pairs $\{(\x_i, \ydelta_i)\}_{i=1,\ldots,N}$, where $\x_i \sim \prior$ and $\ydelta_i = \A \x_i + \eta_i$. In the following, we consider the large-data limit, replacing empirical averages by expectations over the data distribution. We verify generalization to finite datasets numerically.

\section{Preliminaries}
\begin{figure}
  \centering
  \def\svgwidth{\linewidth}
  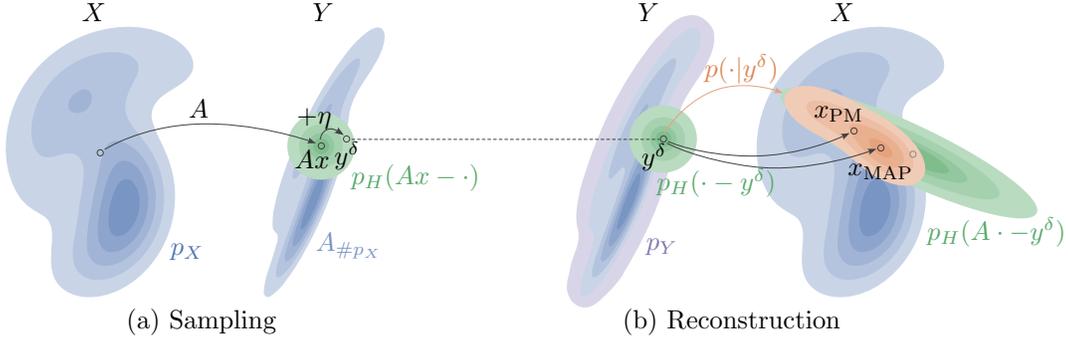
  \caption{The Bayesian inverse problem framework: The forward model maps $\x\sim\prior$ to its measurement, which is corrupted by noise. Given this noisy data, the posterior density is computed as the product of prior and likelihood.}
  \label{fig:bayesian}
\end{figure}

Inverse problems have been solved for many years using classical regularization theory, where the goal is to find regularization operators that approximate the inverse of the forward operator while preserving stability, controlling reconstruction error. In recent years, due to the rise of deep learning and data-based methods, many works take a Bayesian perspective on inverse problems. Figure \ref{fig:bayesian} illustrates this perspective. Similar to the setting that we have introduced above, this framework's key assumptions are that the desired unobserved quantities $\x$ as well as the noise $\eta$ are distributed according to (possibly unknown) probability distributions, here denoted by densities $\prior$ and $\pnoise$. Bayes' law then allows to compute a posterior distribution
\begin{align}\label{eq:posterior}
    p(\x|\ydelta) = \frac{1}{\py(\ydelta)} \pnoise(\A\x-\ydelta)\,\prior(\x),
\end{align}
determining the probability that certain reconstructions $\x$ have led to the measurement $\y$. Here, $\py(\ydelta)=\int_\X  \prior(\x)\,\pnoise(\A\x-\ydelta)\,\diff\x$ is the density of the measured data $\ydelta$. In this framework, regularization methods arise as point estimators of the posterior distribution, with well-known classical methods as special cases.

\subsection{Bayesian Point Estimators}
A popular point estimator is the maximum-a-posteriori (MAP) estimate. In our setting of Gaussian noise, maximizing the posterior (\ref{eq:posterior}) for given $\ydelta$ is equivalent to
\begin{align*}
    \xmap=\arg\min_{\x} \frac{1}{2}\norm{\A\x-\ydelta}^2-\delta^2\log\prior(\x),
\end{align*}
resembling classical variational methods. The first-order optimality condition reads
\begin{align}
    \Aadj(\A\xmap -\ydelta) - \delta^2\nabla_\x (\log\prior)(\xmap) &= 0 \nonumber \\
    \Leftrightarrow \quad (\AadjA - \delta^2\nabla_\x (\log\prior))(\xmap) &= \zdelta, \label{eq:map}
\end{align}
which we can interpret as a normal equation with prior-based regularization of the forward operator. Here, the prior $\prior$ comes into play through its \emph{score function} $\nabla\log\prior$. For a Gaussian prior, this corresponds to the popular classical Tikhonov regularization. Note that this first-order condition is sufficient for the minimizer only if the variational formulation is strictly convex.

Another important point estimator is the posterior mean (PM):
\begin{align*}
    \xpm=\expect_{p(\x|\ydelta)}(\x).
\end{align*}
Extending Tweedie's formula \cite{Efron_2011}, we show that this implies a score-regularized equation in a similar fashion:
\begin{lemma}\label{thm:pm}
    Since we have $p(\x|\ydelta)=p(\x|\zdelta)$ for $\zdelta=\Aadj\ydelta$, the posterior mean for the original problem (\ref{problem:orig}) and the normal equation (\ref{problem:normal}) are equivalent.
    Further, the posterior mean point estimator satisfies
    \begin{align*}
        \A\xpm = \ydelta + \delta^2\nabla_\y(\log\py)(\ydelta).
    \end{align*}
\end{lemma}
\begin{proof}
We start by confirming that $p(\x|\zdelta) = p(\x|\ydelta)$: Let $\ydelta = \ydelta_\parallel + \ydelta_\bot$ be the orthogonal decomposition of $\ydelta$ into $\ydelta_\parallel\in\mathrm{im}(\A)=\ker(\Aadj)^\bot$ and $\ydelta_\bot\in\mathrm{im}(\A)^\bot=\ker(\Aadj)$. Then, we have that
\begin{align*}
    \pnoise(\A\x-\ydelta) &= \frac{1}{(\delta\sqrt{2\pi})^n} \exp\left(-\frac{1}{2\delta^2}\norm{\A\x-\ydelta_\parallel - \ydelta_\bot}^2\right) \\
    &=  \frac{1}{(\delta\sqrt{2\pi})^n} \exp\left(-\frac{1}{2\delta^2}\norm{\A\x-\ydelta_\parallel}^2 + \underbrace{\langle \A\x-\ydelta_\parallel, \ydelta_\bot\rangle}_{=0}- \frac{1}{2\delta^2}\norm{\ydelta_\bot}^2\right) \\
    &= \frac{1}{(\delta\sqrt{2\pi})^n} \exp\left(-\frac{1}{2\delta^2}\norm{\A\x-\ydelta_\parallel}^2\right)\,\underbrace{\exp\left( - \frac{1}{2\delta^2}\norm{\ydelta_\bot}^2\right)}_{=:\kappa(\ydelta_\bot)} \\
    &= \pnoise(\A\x-\ydelta_\parallel) \,\kappa(\ydelta_\bot),
\end{align*}
and therefore
\begin{align*}
    p(\x|\ydelta) &= \frac{\pnoise(\A\x-\ydelta_\parallel)\,\cancel{\kappa(\ydelta_\bot)}\,\prior(\x)}{\int_\X\pnoise(\A\x'-\ydelta_\parallel)\,\cancel{\kappa(\ydelta_\bot)}\,\prior(\x')\,\diff\x'} \\
    &= p(\x|\ydelta_\parallel).
\end{align*}
As a result, we can consider the restriction $\Aadj_\parallel$ of $\Aadj$ to $\mathrm{im}(\A)$, which is an invertible linear operator. All in all, we have $p(\x|\zdelta) = p(\x|(\Aadj_\parallel)\inv\zdelta) = p(\x|\ydelta_\parallel)=p(\x|\ydelta)$.

To show the second part, we extend the proof of Tweedie's formula for Gaussian denoising to inverse problems. First, we note that
\begin{align*}
    \py(\ydelta) &= \int_\X \prior(\x)\, \frac{1}{(\delta\sqrt{2\pi})^n}\exp\!\left(-\frac{1}{2\delta^2}\norm{\A\x-\ydelta}^2\right)\diff\x\\
    &= \int_\X \prior(\x)\, \exp\!\left({-\frac{1}{2\delta^2}\norm{\A\x}^2+\frac{1}{\delta^2}\langle\A\x,\ydelta\rangle}\right)\diff\x\ \frac{1}{(\delta\sqrt{2\pi})^n}\exp\!\left(-\frac{1}{2\delta^2}\norm{\ydelta}^2\right).
\end{align*}
We define
\begin{align*}
    \lambda(\ydelta) := \log\frac{\py(\ydelta)}{\pnoise(\ydelta)} = \log\left\{ \int_\X \prior(\x)\, \exp\!\left({-\frac{1}{2\delta^2}\norm{\A\x}^2+\frac{1}{\delta^2}\langle\A\x,\ydelta\rangle}\right)\diff\x\right\}
\end{align*}
and can compute its gradient in two ways. On the one hand, we have
\begin{align*}
    \nabla_\y\lambda(\ydelta) &= \frac{\int_\X \prior(\x)\, \frac{1}{\delta^2}\A\x\exp\!\left({-\frac{1}{2\delta^2}\norm{\A\x}^2+\frac{1}{\delta^2}\langle\A\x,\ydelta\rangle}\right)\diff\x}{\int_\X \prior(\x)\, \exp\!\left({-\frac{1}{2\delta^2}\norm{\A\x}^2+\frac{1}{\delta^2}\langle\A\x,\ydelta\rangle}\right)\diff\x} \\
    &= \frac{1}{\delta^2}\A \frac{\int_\X \prior(\x)\, \x\exp\!\left({-\frac{1}{2\delta^2}\norm{\A\x}^2+\frac{1}{\delta^2}\langle\A\x,\ydelta\rangle}\right)\diff\x}{\int_\X \prior(\x)\, \exp\!\left({-\frac{1}{2\delta^2}\norm{\A\x}^2+\frac{1}{\delta^2}\langle\A\x,\ydelta\rangle}\right)\diff\x} \\
    &=\frac{1}{\delta^2}\A\xpm,
\end{align*}
where we can exchange integration and differentiation due to the dominated convergence theorem.
On the other hand, it holds
\begin{align*}
    \nabla_\y\lambda(\ydelta) &=  \nabla_\y(\log\py)(\ydelta) - \nabla_\y \left(-\frac{1}{2\delta^2}\norm{\ydelta}^2\right)\\
    &= \nabla_\y(\log\py)(\ydelta) + \frac{1}{\delta^2} \ydelta.
\end{align*}
All in all, we derive $\A\xpm=\ydelta + \delta^2\nabla_\y(\log\py)(\ydelta)$.
\end{proof}

This establishes a clear connection and difference between the two popular point estimators: Instead of regularizing the forward operator, posterior mean corresponds to an explicit denoising step ahead of the reconstruction. Similar to the MAP estimator, the data distribution comes into play through score functions. As depicted in Figure \ref{fig:denoising} for a denoising setting, this can be thought of as implicit or explicit Euler steps of width $\delta^2$ on the scores of $\prior$ respectively $\py$: The MAP estimator moves points along the score vector field of the prior, such that the direction and length of the step match the value of the score at the reconstructed $\xmap$. In contrast, the posterior mean performs a step along the score of the noisy (and therefore smoothed) data distribution, evaluated at the data point $\zdelta$.

\begin{figure}
\centering
\begin{subfigure}[b]{0.45\textwidth}
  \centering
  \def\svgwidth{\linewidth}
  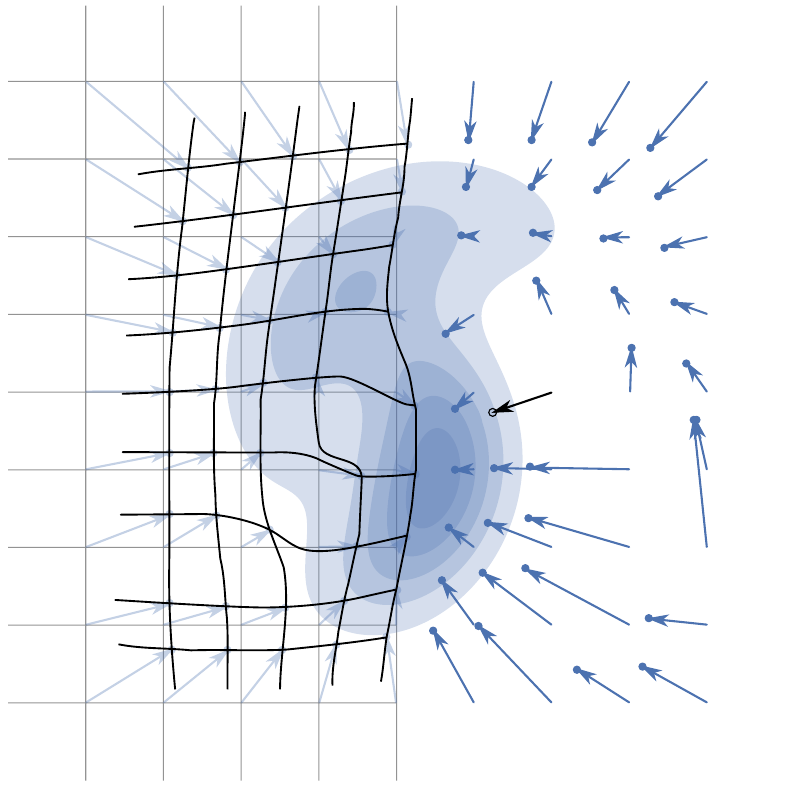
  \caption{MAP estimation (MAP): visualization as an implicit Euler step on the score of the prior distribution $\prior$.}
  \label{fig:denoising_map}
  \end{subfigure}
  \quad
  \begin{subfigure}[b]{0.45\textwidth}
  \centering
      \def\svgwidth{\linewidth}
  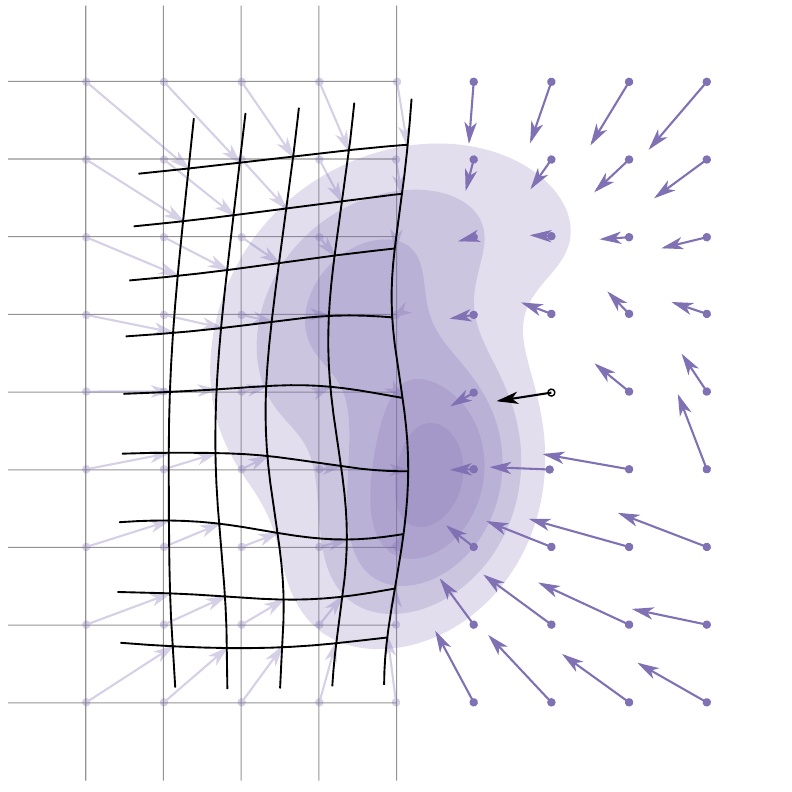
  \caption{Posterior mean estimation: a forward Euler step on the score of the noised distribution $\py$.}
  \label{fig:denoising_pm}
  \end{subfigure}
  \caption{Comparison of MAP and PM for a Gaussian denoising task (i.e. $A=\Id$). The grid visualizes how the data space is deformed by the reconstruction methods. The arrows depict the values of the gradient field at the connected dots – for implicit Euler, we therefore plot the gradients at the reconstructed points.}
  \label{fig:denoising}
\end{figure}

\subsection{Reconstruction by Invertible Neural Networks}
In previous work \cite{Arndt_2023,Arndt_2024}, different training strategies to optimize iResNets, a subclass of invertible neural networks, have been studied theoretically and numerically. From a Bayesian perspective, the authors studied two approaches: approximation training and reconstruction training.

It has been shown that the optimal solution of approximation training
\begin{align}\label{eq:approx_loss}
    \modelapprox = \arg\min_{\model} \expect_{(\x,\zdelta)\sim p(\x,\zdelta)}\left(\frac{1}{2}\norm{\model(\x)-\zdelta}^2\right)
\end{align}
is, in the unconstrained case, $\modelapprox=\AadjA$, and if the set of possible $\modelapprox$ is constrained to iResNets, the learned regularization scheme does at most depend on the first (and second) moments of the prior $\prior$, while there is no influence of the noise which is present in the training data. Note that these results also extend to the approximation of $\x\mapsto\ydelta$, where the minimizer equals $\A$ itself. Motivated by these limitations, we aim to study regularization terms in approximation training. Can these improve the data-dependency?

One popular approach towards data dependence that has also been studied for invertible networks in \cite{Arndt_2024} is to approximate the reconstruction mapping directly: For $\invmodel=(\model)\inv$, reconstruction training
\begin{align*}
    \modelreco = \arg\min_{\model} \expect_{(\x,\zdelta)\sim p(\x,\zdelta)}\left(\frac{1}{2}\norm{\x-\invmodel(\zdelta)}^2\right)
\end{align*}
approximates the posterior mean, if the PM estimator is invertible and the network architecture is sufficiently expressive. In this case, a data-dependent structure is being learned, with the degree of regularization steered by the noise level in the dataset. However, this approach loses explicit control of the forward operator approximation. Once more, this also holds for the original problem $\x\mapsto\ydelta$, where we have shown the equivalence of the posterior mean estimators in Lemma \ref{thm:pm}.
Our goal is to bridge the gap between these two extremes by augmenting approximation training with penalties that introduce data-dependency – retaining explicit forward‑map fidelity while also enabling controllable regularization.

\section{Related Work}
This work builds upon the two previously mentioned publications by Arndt et al. \cite{Arndt_2023, Arndt_2024}. While we point to the existing literature – for example, on existing invertible network architectures – where it is relevant, the aim of this section is to provide a concise introduction into three aspects: the application of invertible neural networks to inverse problems, regularization approaches in neural network training, and the analysis of Bayesian point estimators.

Invertible networks often appear in the context of generative modeling by Normalizing Flows \cite{Dinh_2015}, which can be applied to inverse problems, for example, to learn the prior distribution \cite{Park_2024, Whang_2021} or directly as Conditional Normalizing Flows to sample from the posterior \cite{Govinda_2021, Denker_2021, Kaltenbach_2023}. The latter can also be formulated by incorporating the conditioning into the normalizing distribution \cite{Izmailov_2020}, or by adding latent dimensions in the input and output spaces of an invertible network, which requires bi-directional training \cite{Ardizzone_2019}. Beyond statistical estimates, invertible networks have also been shown to yield competitive, interpretable and stable point reconstructions \cite{Arndt_2025}, where the interpretability arises from the simultaneous approximation of the forward and inverse problems. Our work extends on this line of research by introducing new loss functions that achieve small errors for both directions: forward operator approximation as well as reconstruction.

Regularization of neural networks can be pursued with a variety of goals and methods. Note that in our context, there is some ambiguity in the term \emph{regularization}: while it always refers to stabilizing the solution of underdetermined problems, the specific problem in focus differs. In deep learning, regularization typically aims to improve the generalization of the learned network to unseen data or new tasks \cite{Goodfellow_2016}. In contrast, regularization in inverse problems refers to incorporating prior knowledge to produce stable and meaningful reconstructions \cite{Benning_Burger_2018}. We operate at the intersection of both perspectives, using regularization terms during training to influence the learned reconstruction (regularization) operator for the inverse problem. In deep learning, the methods include architectural constraints, adapted loss functions, and data preprocessing \cite{Goodfellow_2016}. While popular loss terms often constrain the norm of the network weights directly, there are also methods that act on the Jacobian of the network \cite{Behrmann_2021, Jakubovitz_2018}. The log-determinant of the Jacobian appears naturally in Normalizing Flows \cite{Dinh_2015}, while the trace of the Jacobian (the divergence of the vector field defined by the network) is central in implicit score matching for Diffusion Models \cite{hyvarinen_2005} and the related continuous  Normalizing Flows  \cite{Chen_2018}. Although our work focuses on point estimators for inverse problems, the study of these approaches as regularization terms also contributes to their understanding in the broader context of Normalizing Flows and Diffusion Models.

A rigorous and extensive introduction to Bayesian inverse problems is provided by Arridge et al. \cite{Arridge_Maass_Öktem_Schönlieb_2019}, where various reconstruction methods are derived as different posterior point estimators. Differences between estimators, such as MAP and posterior mean, have also been discussed by Calvetti et al. \cite{Calvetti_2007}.

\section{Theoretical Results}
As motivated initially, our goal is to design training strategies that (i) compel an invertible network $\model$ to approximate the forward operator and (ii) simultaneously impose data‐dependent regularization on the reconstruction via $\invmodel$. A promising approach towards bi-directional mappings is to train Conditional Normalizing Flows (CNFs) that include the conditioning as a parameter in the base distribution, such as in FlowGMM \cite{Izmailov_2020}. In our inverse-problem setting and if $\Y=\X$ (as we will assume for Section \ref{sec:logdet}), one might use a conditioned base distribution $\hat{p}(\cdot|\ydelta):=\mathcal{N}(\ydelta,\deltachosen^2)$ for some \textit{chosen} parameter $\deltachosen$ and train the network $\model$ by maximum likelihood such that for $\x_i\sim\prior$, $\model(\x_i)$ follows the corresponding base distribution $\hat{p}(\cdot|\ydelta_i)$, approximating the forward mapping $\x\mapsto\ydelta$.
Ideally, the learned inverse mapping could then serve to sample from the posterior by first adding Gaussian noise to a given $\ydelta$. 

However, in practice, the added noise has only a local effect – nearby $\ydelta$ values remain strongly coupled – so the CNF cannot capture the full diversity of the true posterior $p(\x|\ydelta)$. Thus, exact posterior sampling fails. Nevertheless, we expect that this approach imposes some useful data dependency on the learned model. This leads us to ask:
\emph{Which point estimator does this CNF training actually produce?}

To do this, we change our perspective on this approach and reinterpret it as an augmentation of the approximation loss (\ref{eq:approx_loss}) via a regularization term. We prove and discuss the properties of the resulting reconstructor and then proceed to study another variant of this approach, proving connections to known point estimators.

\subsection{Regularization by Log-Jacobian-Determinant}\label{sec:logdet}
The maximum-likelihood optimization problem for the previously described flow-based approach reads
\begin{align}\label{eq:logdet_loss}
    \modelogdet = \arg\min_{\model} \expect_{(\x,\ydelta)\sim p(\x,\ydelta)}\left(\frac{1}{2}\norm{\model(\x)-\ydelta}^2 - \deltachosen^2\, \log|\det\Diff\model(\x)|\right).
\end{align}
Although motivated by CNFs, this formulation can be interpreted more generally: as a standard forward-approximation loss augmented by a volume-expansion penalty.  Here, the chosen $\deltachosen^2$ plays the role of a regularization weight: the network is driven not only to fit the operator $\x \mapsto \ydelta$ but also to \emph{locally expand} high-density regions, thereby counteracting the contraction imposed by an ill-posed operator $\A$.  As a result, the learned inverse $\invmodelogdet$ naturally pushes reconstructions toward high-density areas of the prior, yielding a built-in regularizing effect. A precise characterization of this effect follows from the next theorem.

\begin{theorem}\label{thm:logdet}
    The minimizer of the log-determinant-regularized loss in (\ref{eq:logdet_loss}) satisfies
    \begin{align*}
        \modelogdet(\x) = \A \x - \deltachosen^2\,\nabla_\y(\log \priorpushed)(\modelogdet(\x)),
    \end{align*}
    where $\priorpushed$ denotes the density of the push-forward of the prior $\prior$ through $\modelogdet$.
\end{theorem}

\begin{proof}
    We start by noting, similar to \cite{Arndt_2024}, that the target is not influenced by the noise:
    \begin{align*}
    &\ \expect_{(\x,\ydelta)\sim p(\x,\ydelta)}\!\left(\frac{1}{2}\norm{\ydelta - \model(\x)}^2 - \deltachosen^2\,\log|\det\Diff\model(\x)|\right) \\
    =&\ \expect_{\x\sim\prior, \eta\sim\pnoise}\!\left(\frac{1}{2}\norm{\A\x+\eta - \model(\x)}^2 - \deltachosen^2\,\log|\det\Diff\model(\x)| \right)\\
    =&\ \expect_{\x\sim\prior, \eta\sim\pnoise}\!\left(\frac{1}{2}\norm{\A\x\! -\! \model(\x)}^2\! -\langle\eta,\A\x-\varphi(\x)\rangle +  \frac{1}{2}\norm{\eta}^2\!\!- \deltachosen^2\log|\det\Diff\model(\x)| \right)\\
    =&\ \expect_{\x\sim \prior}\!\left(\frac{1}{2}\norm{\A\x - \model(\x)}^2 - \deltachosen^2\,\log|\det\Diff\model(\x)| \right) + C_1
    \end{align*}
    since $\langle\eta,\A\x-\varphi(\x)\rangle$ has zero expectation (due to the independence of $\x$ and $\eta$) and $C_1:=\expect_{\eta\sim\pnoise}\frac{1}{2}\norm{\eta}^2$ is constant and independent of the optimization.
    
    To show the property of the minimizer $\model$, we use first-order optimality: Given a test function $h\in C^\infty_c(\X,\X)$, since $\Diffeo^2(\supp(h))$ is open in $C^2(\supp(h), \supp(h))$, it exists $\varepsilon>0$ small enough such that $\model+\varepsilon h\in\Diffeo^2(\X)$.
    The minimizer has to satisfy
\begin{align*}
    \frac{\diff}{\diff \varepsilon}\left( \expect_{\x\sim\prior} \left(\frac{1}{2}\norm{\model(\x)+\varepsilon h(\x)-\A\x}^2 - \deltachosen^2\log\left(\left|\det\Diff\model(\x)+\varepsilon\Diff h(\x)\right|\right) \right) \right)_{\varepsilon=0}&=0,
\end{align*}
where the dominated convergence theorem together with the compact support of $h$ allow to pull the differentiation into the integral:
\begin{align*}
    \expect_{\x\sim\prior} \left(\langle\model(\x)-\A\x, h(\x)\rangle - \deltachosen^2\tr\left(\Diff\model(\x)^{-1}\Diff h(\x) \right)\right)&=0\\
    \Leftrightarrow \int_\X \prior(\x) \left(\langle\model(\x)-\A\x, h(\x)\rangle - \deltachosen^2\tr\left(\Diff\model(\x)^{-1}\Diff h(\x) \right)\right) \,\diff\x&=0.
\end{align*}
To decompose the last term, one can leverage the product rule for the divergence reversely. This leads to
\begin{align*}
    \tr\left(\Diff\model(\x)^{-1}\Diff h(\x) \right) &= \nabla\cdot(\Diff\model(\x)^{-1} h(\x)) - \langle \nabla\cdot(\Diff\model(\x)^{-1}), h(\x)\rangle.
\end{align*}

On the resulting integral, we can apply integration by parts, where the boundary terms vanish since $h$ has compact support:
\begin{align*}
    \int_\X \prior(\x) \Bigl(\langle\model(\x)-\A\x, h(\x)\rangle & - \deltachosen^2(\nabla\cdot(\Diff\model(\x)^{-1} h(\x)) - \langle \nabla\cdot(\Diff\model(\x)^{-1}), h(\x)\rangle) \Bigr) \,\diff\x \\
    \begin{split}
        =\int_\X \prior(\x)\langle\model(\x)-\A\x, h(\x)\rangle &+ \deltachosen^2\langle\Diff\model(\x)^{-1} h(\x), \nabla \prior(\x)\rangle\\
    &+\deltachosen^2\langle \prior(\x)\nabla\cdot(\Diff\model(\x)^{-1}), h(\x)\rangle  \,\diff\x.
    \end{split}
\end{align*}
Now, by the fundamental lemma of the calculus of variations, we can omit $h$ and the integral and write
\begin{align*}
    (\model(\x)-\A\x) + \deltachosen^2 \Diff\model(\x)^{-T} \frac{\nabla \prior(\x) }{\prior(\x)}+ \deltachosen^2\nabla_\x\cdot(\Diff\model(\x)^{-1})&= 0.
\end{align*}
To proceed, we apply the inverse function rule. We have:
\begin{align*}
    \Diff \model(\x)\inv = \Diff \invmodel(\model(\x)) \quad \text{and} \quad \Diff \invmodel(\y)\inv = \Diff \model(\invmodel(\y)).
\end{align*}
We further define $\y:=\model(\x)$ and can conclude
\begin{align*}
    (\model(\x)-\A\x) + \deltachosen^2 \Diff \invmodel(\model(\x))^{T} \frac{\nabla \prior(\x) }{\prior(\x)}+ \deltachosen^2\nabla_\x\cdot(\Diff \invmodel(\model(\x)))&= 0\\
    \Leftrightarrow (\model(\x)-\A\x) + \deltachosen^2 \Diff \invmodel(\model(\x))^{T} \frac{\nabla \prior(\x) }{\prior(\x)}+ \deltachosen^2\Diff\model(\x)^T (\nabla_\y\Diff \invmodel)(\model(\x))&= 0 \\
    \Leftrightarrow (\y-\A\invmodel(\y)) + \deltachosen^2 \Diff \invmodel(\y)^{T} \frac{\nabla \prior(\invmodel(\y)) }{\prior(\invmodel(\y))} + \deltachosen^2 \Diff\invmodel(\y)^{-T}\nabla_\y\Diff\invmodel(\y)&= 0\\
    \Leftrightarrow (\y-\A\invmodel(\y)) + \deltachosen^2 \nabla_\y\log\prior(\invmodel(\y)) + \deltachosen^2 \nabla_z\log|\det\Diff\invmodel(\y)|&= 0\\
    \Leftrightarrow (\y-\A\invmodel(\y)) + \deltachosen^2 \nabla_\y\left(\log(\prior(\invmodel(\y))|\det\Diff\invmodel(\y)|)\right)&= 0\\
    \Leftrightarrow (\y-\A\invmodel(\y)) + \deltachosen^2 \nabla_\y(\log\model_{\#}\prior)(\y) &= 0.
\end{align*}
Rephrasing this in terms of the learned forward model $\modelogdet$ yields the desired result.
\end{proof}

First, we notice the clear data dependence induced by this approach: The volume-expansion term injects prior information into the forward model approximation, visible through the appearance of $\prior$. Second, we emphasize that the resulting network is independent of the noise level $\delta$ in the training data; instead, the degree of regularization depends solely on the prior and the chosen $\deltachosen$. 
If we denote the reconstruction of $\ydelta$ by $\xlogdet$, we can read this formula as
\begin{align*}
    \A\,\xlogdet = \ydelta + \deltachosen^2\,\nabla_\y(\log\priorpushed)(\ydelta).
\end{align*}
Notice that this generally resembles well-known ingredients: Similar to the condition for posterior mean estimation, the reconstruction can be thought of as a denoising before the backprojection of the data. The denoising is once more given by a single forward Euler step on a score vector field. In this case, however, this density is only implicitly defined as the push-forward of the prior through the learned forward model. We further analyze the reconstruction method in the following remark.
\begin{remark}
Our goal is to study the relation between the reconstruction method induced by the log-determinant regularized model and the posterior mean:
\begin{alignat*}{3}
    \A\,&\xlogdet &= &\, \ydelta+\deltachosen^2\nabla_\y (\log\priorpushed)(\ydelta) \\
    \A\,&\xpm &= &\, \ydelta + \delta^2\,\nabla_\y(\log\py)(\ydelta).
\end{alignat*}
To do this, we analyze $\py$ and compare it to $\priorpushed$: We have that $\py=(\A_\#\prior)\ast\pnoise$. Since the convolution with a Gaussian corresponds to the solution of a diffusion process $\partial_t p_t=\frac{1}{2}\Delta p_t$, $p_0=\A_\#\prior$ at time point $\delta^2$, we can consider the corresponding probability flow. According to the Liouville equation, the trajectories of
\begin{align*}
    \partial_t y_t &= -\frac{1}{2}\nabla \log p_t(y_t)\\
    y_0 &= \A\x
\end{align*}
share the same time evolution of distributions $p_t$. We denote the flow map for this ODE by \smash{$F_t:\Y\to\Y$}, it fulfills
\begin{align*}
    \partial_t F_t(\A\x) = -\frac{1}{2}\nabla (\log F_{t\#}\A_\#\prior)(F_t(\A\x)).
\end{align*}
The implicit Euler for time step $\Delta t$ reads
\begin{align*}
   \hat{F}_{\Delta t}(\A\x) &= \A\x -\Delta t\,\frac{1}{2}\nabla (\log \hat{F}_{{\Delta t}\#}\A_\#\prior)(F_{\Delta t}(\A\x)).
\end{align*}
If we set $\model:= \hat{F}_{\Delta t}\circ\A$ and $\Delta t:=2\deltachosen^2$, this corresponds to the equation
\begin{align*}
    \A\x &= \model(\x) + \deltachosen \,  \nabla(\log \model_{\!\!\!\#}\prior)(\model(\x)),
\end{align*}
which is exactly the above condition on the minimizer $\modelogdet$. Thus, we have constructed a model that fulfills this equation.
With $\deltachosen=\delta$, the score step on $\priorpushed$ therefore has similar properties to $\py$, but acts on a distribution that is evolved with an implicit Euler step (instead of a time continuous Gaussian process) and a doubled step size, leading to a stronger smoothing of the score field.

\end{remark}

As a result, we expect the reconstruction with $\invmodelogdet$ to resemble a smoothed-out version of the posterior mean. In contrast to reconstruction training, the amount of regularization is steered by the choice of $\deltachosen$ and does not directly depend on the data noise. This should solve one of our designated tasks: it provides a method that can explicitly balance forward operator approximation and data-based reconstruction accuracy. We nevertheless recall that this derivation was limited to the case of $\Y=\X$.
As stated initially, we additionally search for methods that allow for high-fidelity reconstructions, and aim to modify this approach to learn the MAP estimate.

\subsection{Regularization by Divergence}
Building upon the proof of Theorem \ref{thm:logdet}, we tackle this challenge by a second approach. Here, we employ the more general normal equation (\ref{problem:normal}) problem and a linearization of the log-determinant: the vector-field divergence of the model $\nabla\cdot\model(\x)$. The resulting minimization problem therefore reads
\begin{align}\label{eq:div_loss}
    \modeldiv = \arg\min_{\model} \expect_{(\x,\zdelta)\sim p(\x,\zdelta)}\left(\frac{1}{2}\norm{\model(\x)-\zdelta}^2 - \deltachosen^2\, \nabla\cdot\model(\x)\right).
\end{align}

The intuition behind the divergence remains similar to the previous approach: The regularization objective can be equivalently seen as maximizing the divergence of $\model-\Id$, which describes how points move in the learned (regularized) forward mapping. Roughly speaking, the vector-field divergence of the vector field determines how much the mapping is expanding or shrinking at each point. While the approximation of an ill-posed $\A$ forces the model to contract the space, maximizing the divergence for regions with high likelihood means that these regions are less contracted.

The following theorem demonstrates that this leads to the approximation of a data-based regularization of $\AadjA$ -- we extract the forward operator that appears in the regularized normal equation from first-order optimality for MAP estimation (\ref{eq:map}). Once more, $\deltachosen$ acts as a regularization parameter – in line with the MAP estimate, we will usually choose it as the noise level for which we aim to reconstruct, i.e. $\deltachosen=\delta$.

\begin{theorem}\label{thm:div} The divergence-regularized loss in (\ref{eq:div_loss}) is equivalent to 
\begin{align*}
    \expect_{\x\sim\prior} \left(\frac{1}{2}\norm{\model(\x) -\left(\AadjA\x - \deltachosen^2 \nabla_\x (\log\prior)(x)\right)}^2\right).
\end{align*}
Therefore, the optimal solution for unconstrained $\model$ is $\modeldiv=\AadjA - \deltachosen^2 \nabla_\x (\log\prior)$.
\end{theorem}
\begin{proof}
    Similar to Theorem \ref{thm:logdet}, it follows that the objective is independent of the noise.
    We proceed by proving the equivalence of this loss function: We have
    \begin{align*}
    &\ \expect_{\x\sim\prior}\! \left(\frac{1}{2}\norm{\AadjA\x-\deltachosen^2 \nabla_\x (\log\prior)(x)-\model(\x)}^2\right)\\
    =&\ \expect_{\x\sim\prior}\! \left(\frac{1}{2}\norm{\AadjA\x-\model(\x)}^2\! + \frac{1}{2}\norm{\deltachosen^2\nabla_\x (\log\prior)(x)}^2\!\! - \langle \AadjA\x-\model(\x), \deltachosen^2 \nabla_\x (\log\prior)(x) \rangle\right)\\
    =&\ \int_\X \prior(\x) \left(\frac{1}{2}\norm{\AadjA\x-\model(\x)}^2 + C - \deltachosen^2\, \left\langle \AadjA\x-\model(\x), \frac{\nabla_\x \prior(x)}{\prior(\x)} \right\rangle\right)\,\diff\x \\
    =& \int_\X \prior(\x) \frac{1}{2}\norm{\AadjA\x-\model(\x)}^2 - \deltachosen^2 \left\langle \AadjA\x-\model(\x), \nabla_\x \prior(x) \right\rangle\,\diff\x + C.
\end{align*}
For the second term, we can employ multi-dimensional partial integration,
where the surface integral vanishes since we assume the expectation $\int_\X\prior(\x)(\AadjA\x-\model(\x))\, \diff\x$ to exist. We can summarize
\begin{align*}
    &\ \int_\X \prior(\x) \frac{1}{2}\norm{\AadjA\x-\model(\x)}^2 - \deltachosen^2 \left\langle \AadjA\x-\model(\x), \nabla_\x \prior(x) \right\rangle\,\diff\x\\
    =& \int_\X \prior(\x) \frac{1}{2}\norm{\AadjA\x-\model(\x)}^2 + \prior(\x)\,\deltachosen^2\, \nabla \cdot (\AadjA\x-\model(\x)) \,\diff\x \\
    =& \int_\X \prior(\x) \frac{1}{2}\norm{\AadjA\x-\model(\x)}^2 - \prior(\x)\, \deltachosen^2\,\nabla \cdot \model(\x) \,\diff\x + \deltachosen^2\,\tr(\AadjA)\\
    =&\ \expect_{\x\sim\prior} \left(\frac{1}{2}\norm{\AadjA\x-\model(\x)}^2 - \deltachosen^2\,\nabla\cdot\model(\x)\right) + C_3,
\end{align*}
which shows the desired equivalence. Note that this proof uses similar techniques as the study of implicit score matching losses \cite{hyvarinen_2005}.
\end{proof}

This unveils once more a type of derivative-based regularization term that leads to a data-dependent regularization for the resulting reconstructor. As before, the regularization degree is controlled by the choice of $\deltachosen$, instead of the data noise level $\delta$ itself. In this case, we are able to not only show a condition for the minimizer but also derive the equivalence of the loss function to a simple expectation of the squared distance. As a result, the regularizing effect comes into play independent of the optimization strategy.

Since the approximated operator exactly resembles the necessary condition for MAP estimation (\ref{eq:map}), we argue that the divergence regularization term generally leads to highly desirable data-based point estimators and later underline this hypothesis by first numerical experiments. However, in many cases, the MAP estimator itself is not invertible and can therefore not be expressed by our proposed methods. A special set of problems, for which the first-order optimality is also a sufficient condition, is log-concave priors.
\begin{corollary}
    Let $\prior$ be a strongly log-concave and differentiable prior density. Then, the MAP estimation problem is strongly convex, and the first-order optimality condition becomes a sufficient condition. In addition, the regularized forward operator is invertible \cite{Arndt_2024}, and therefore the reconstruction method given by the minimizer of (\ref{eq:div_loss}) is equivalent to MAP estimation.
\end{corollary}
In the appendix of \cite{Arndt_2024}, the authors show how this limitation can be loosened in the direction of large singular values of $\A$ or small noise levels. Especially for the denoising problem, e.g. in proximal operators, this also enables us to learn MAP estimates for non-convex negative log-likelihoods in the prior. In this case, inverting the network corresponds to solving the equation for an implicit Euler scheme, as it is depicted in Figure \ref{fig:denoising}.

\section{Numerical Experiments}
To underline our findings and study the behavior of the presented regularization terms in practice, we perform first numerical experiments on a two-dimensional toy problem\footnote{\url{https://gitlab.informatik.uni-bremen.de/inn4ip/regularization-terms-for-inn}}.
In our computations, we use iResNets \cite{behrmann_2019}, a class of neural networks that achieve invertibility through monotonicity, enforced by a Lipschitz constraint during training, see Section \ref{sct:iresnets}. The inverse of these networks is given by a fixed point iteration with guaranteed convergence. Nevertheless, our theoretical results are not limited to these and can be extended to other approaches such as architectures based on affine coupling blocks \cite{dinh_2017, kingma_2018}. 

\subsection{Implementation}
\subsubsection{Efficient Computation of Regularization Terms}
For low-dimensional problems, the vector-field divergence as well as the log-determinant of the Jacobian can be easily computed via backpropagation. However, since a separate backward pass for each dimension is necessary, the exact computation of those values is infeasible for high-dimensional settings, e.g. in imaging.

A common way to overcome this in the computation of the divergence is to use Hutchinson's trace estimator on the Jacobian \cite{Hutchinson_1989}. The idea behind this approach is to compute
\begin{align*}
    \tr(\Diff\model(\x))=\expect_{\varepsilon\sim\mathcal{N}(0,1)^{n_\X}} \langle \varepsilon, \Diff\model(\x)\, \varepsilon\rangle
\end{align*}
via a finite set of samples $\varepsilon$ for every sample $\x$. The matrix-vector product $\Diff\model(\x)\, \varepsilon$ can be efficiently computed in a single backward pass. In practice and for larger batch sizes, we find a single sample $\varepsilon$ per element in the batch to be sufficient for a good estimation of the divergence in the loss. As a result, this approach can be applied very efficiently for every deep neural network, without any requirements on the architecture.

The computation of the log-determinant is a usual problem in Normalizing Flows and has therefore been solved for most invertible neural network architectures. In iResNets \cite{behrmann_2019}, the authors use an approach that is also based on Hutchinson's trace estimator but approximates a truncated power series to compute the matrix logarithm in $\tr\log\Diff\model_{\!\!\!\!i}=\log|\det\Diff\model_{\!\!\!\!i}|$ for every layer individually. The log-determinant of the full network is then given as the sum of the individual layers' log-determinants.

\subsubsection{iResNets} \label{sct:iresnets}
To offer a guaranteed and stable inverse, the idea behind iResNets is to interpret residual networks as Euler discretizations of ODEs. As a consequence, one can leverage that those offer a guaranteed backward time if the right-hand side is Lipschitz continuous and the time steps are sufficiently small. Formally, the neural network layer
\begin{align*}
    \model_{\!\!\!i} = \Id - f_i
\end{align*}
is invertible if $\Lip(f_i)\leq L<1$ for some arbitrary and possibly non-linear residual block $f_i$. The inverse of a layer can be computed by a fixed point iteration that is guaranteed to converge by the Banach fixed point theorem. iResNets employ multiple of these layers sequentially. They have recently been shown to offer competitive performance for inverse problems \cite{Arndt_2025}. In our experiments with the reconstruction loss, the fixed point iteration is necessary in every training step. While 10-20 iterations are enough to obtain small reconstruction errors for most cases (depending on the choice of $L$), we usually perform 50-100 to guarantee comparability to the other approaches during our studies. To keep the amount of used memory feasible, we employ ideas from deep equilibrium models \cite{Bai_2019} for the backward pass.

\begin{figure}
    \centering
    \begin{subfigure}[b]{0.4\textwidth}
        \includegraphics[width=\textwidth]{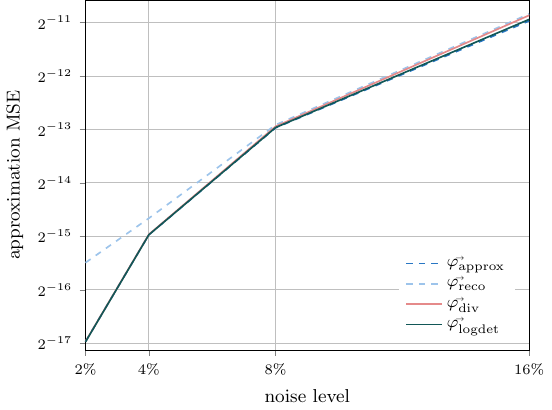}
    \end{subfigure}
    \quad
    \begin{subfigure}[b]{0.4\textwidth}
        \includegraphics[width=\textwidth]{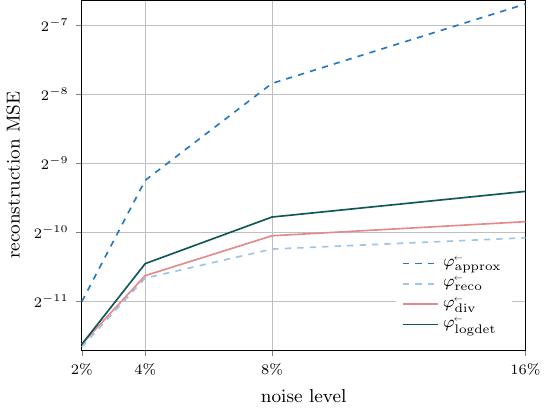}
    \end{subfigure}
    \caption{Reconstruction and forward operator approximation error (including noise) on the 2d dataset for $\varepsilon=\frac{1}{8}$: the divergence-regularized loss obtains small errors in both directions.}
    \label{fig:mse}
\end{figure}

\subsection{Interpretable regularization for a 2d problem}
To better understand the behavior of the models that result from the regularized training and verify our analytical results numerically, we study a 2-dimensional bimodal dataset. The prior $\prior$ is constructed from two polar transformed Gaussians and is identical to the distribution used for visualization purposes in the previous sections. We emphasize that the used prior is not log-concave, therefore we can use it to study the behavior of the divergence-regularization training beyond the theoretical guarantees for MAP estimation. We perform experiments for the case of denoising $\A=\Id$ as well as a series of forward operators
\begin{align*}
    \A_\varepsilon=\begin{pmatrix}
        1 & 1 \\ 1 & 1+\varepsilon
    \end{pmatrix},
\end{align*}
which are increasingly ill-posed for $\varepsilon\to0$ (cf. \cite{Maass_2019}). We further use several noise levels $\delta$ on this data during training and choose $\deltachosen=\delta$ for the two regularized loss functions.\footnote{During experiments, we verified the theoretical result that the data noise level does not influence the training outcome except for reconstruction training. To be consistent among the experiments, we nevertheless add noise for all training strategies.} We sample $4\times10^5$ points for training and $1\times10^5$ points from the same distribution for testing and use a network consisting of three invertible residual blocks, each with a Lipschitz constant of $L=0.99$ and consisting of two fully-connected hidden layers. 
For the 2d setting, the exact computation of log-determinant and divergence via backpropagation remains feasible -- therefore, we use these instead of their previously discussed statistical estimates.

\begin{figure}
  \centering
  \includegraphics[width=\textwidth]{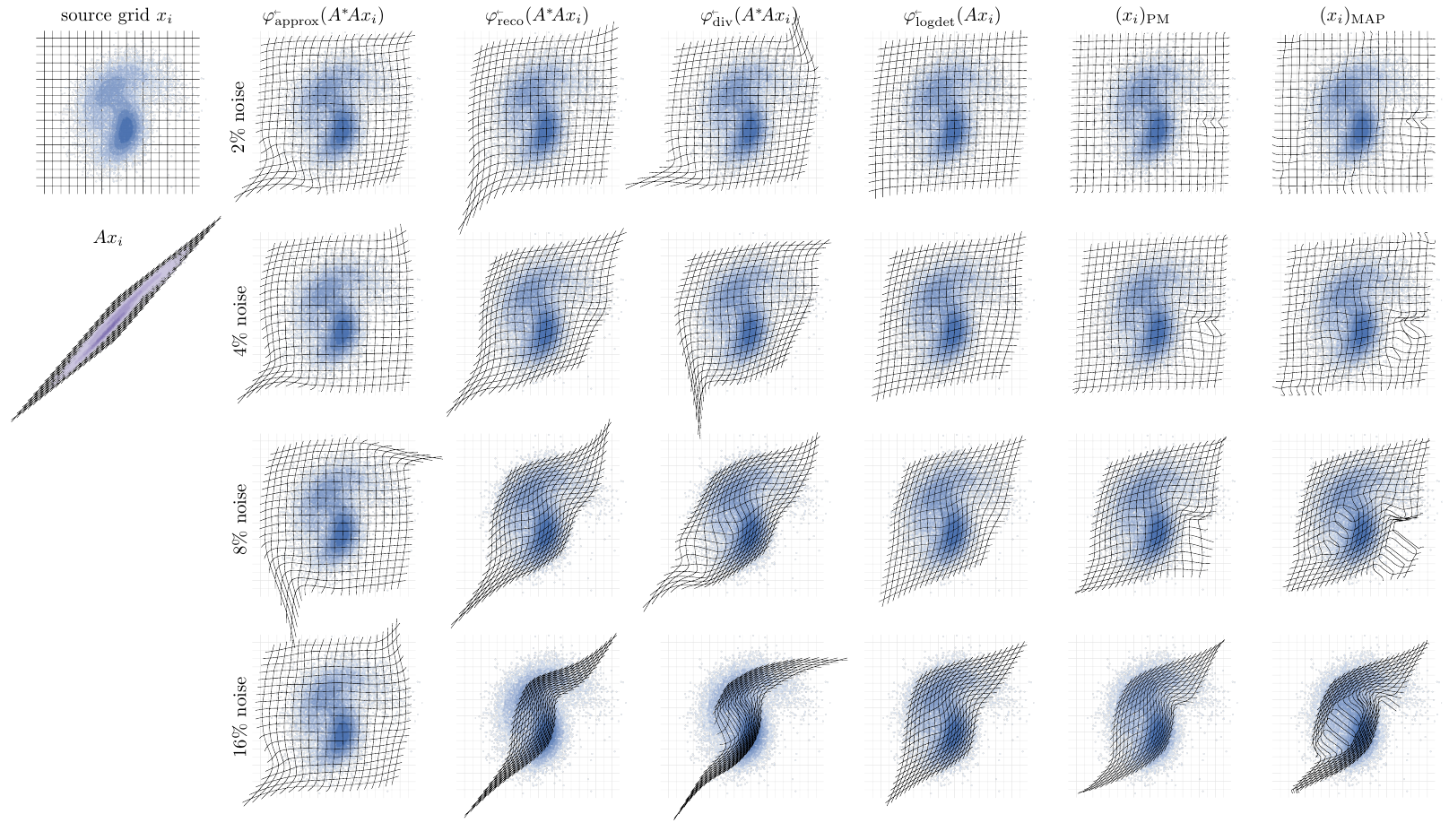}
  \caption{Grids reconstructed by the optimized networks for $\varepsilon=\frac{1}{2}$ and $\deltachosen=\delta$, depicted alongside the numerically computed posterior mean and MAP estimators. (For optimal visibility of details, please view the figure electronically and zoom in as needed.)}
  \label{fig:grids}
\end{figure}

For the ill-posed operator $\A_\frac{1}{8}$, we compare the errors for forward and inverse problem, see Figure \ref{fig:mse}. The invertibility of the architecture allows us to evaluate both, offering insights into the deviation of the learned reconstruction method from the known unstable forward operator. We notice that for all cases, the regularized approximation losses greatly improve the reconstruction error upon the standard approximation loss, while preserving approximation accuracy. Since the reconstruction MSE and approximation MSE match exactly the objectives for $\modelreco$ respectively $\modelapprox$, these models consistently perform best in these metrics. However, especially for small noise levels, both regularization methods performs almost equally well in terms of both, data consistency and reconstruction. We can therefore note that the divergence-regulation is a promising solution to our search for bi-directional training strategies.

Since we chose a 2d problem, we can interpret the learned methods by visualizing the reconstruction mappings alongside $\prior$. For the experiments with $\A_\frac{1}{2}$, we depict the models $\invmodel:\R^2\to\R^2$ by plotting the deformation of a grid through $\invmodel\circ\AadjA$ respectively $\invmodel\circ\A$ for the log-determinant regularization. This way, the grids become more dense in regions where the model is more likely to reconstruct, while the non-regularized inverse operator reconstructs the original equidistant grid. Therefore, the visualization allows an intuitive interpretation of the used regularization terms. The results are shown in Figure \ref{fig:grids}.

Generally, the action of $\deltachosen$ as a regularization parameter becomes visible: for small noise levels, the operator is well approximated, which implies a barely distorted grid. The contraction of the grid increases for increased $\deltachosen$, especially in the direction of the smaller singular value of $\A_\varepsilon$.
An especially interesting comparison is between the two regularized networks: In comparison with the log-determinant regularized network, the divergence-regularized network pushes the points explicitly towards the peaks of the prior, visible via a less dense grid between both peaks. This mirrors the properties of MAP in contrast to the posterior mean, which confirms our theoretical analysis and provides additional indication that the divergence-based regularization leads to high-quality reconstructions.
Meanwhile, $\modelogdet$ generally shows properties that are similar to the posterior mean, but creates smoother grids. This aligns with the theory in Remark \ref{thm:pm}.

As predicted by the theory, without regularization, the reconstructed grid is not influenced by the noise level and does not show any data-based regularization. Note that the model is only optimized where samples are located, and therefore there is some undetermined behavior in the less likely regions.

These results are validated and extended by the denoising problem in Figure \ref{fig:first_page_grids}: While approximation training implies no regularization, the regularization terms pull the reconstructed points towards the higher-density regions. Again, the MAP-like behavior of $\modeldiv$ is clearly visible. As discussed in Remark \ref{thm:pm}, we expect the log-determinant regularization to be close to the posterior mean, but act as a score step on a smoother distribution – this fits well with the visual comparison of $\modelogdet$ and $\modelreco$.

\section*{Conclusion}
We have introduced and analyzed two novel, interpretable regularization schemes for invertible neural networks in linear inverse problems. Building on the limitations of prior iResNet-based methods \cite{Arndt_2023,Arndt_2024}, we distilled insights from classical Bayesian point estimators to craft regularization terms for loss functions that both preserve forward‐map fidelity and inject principled, data‐dependent regularization.

Our first regularizer – a log‐Jacobian‐determinant penalty – draws directly from Normalizing Flow theory and yields a score‐based correction akin to posterior‐mean denoising. While effective, its construction for identical input and output spaces together with the introduced smoothing of the data distribution limit its overall applicability. To overcome this, we proposed a divergence‐based penalty that similarly promotes expansiveness in regions of high data density. We have shown analytically and numerically that it resembles properties of the MAP estimator and even coincides with it in convex settings. The resulting approach therefore allows to approximate MAP estimation by a pass through an invertible neural network, admitting stability as well as interpretability. To get a first and intuitive impression of the consequences of our regularization terms in practice, we study two-dimensional inverse problems, where we can visually understand the characteristics of different reconstruction methods. The results verify our developed theory.

Although our results are primarily theoretical, they derive an entirely new concept and therefore open several promising directions for future work. From an applied perspective, it would be valuable to extend divergence‑based regularization to nonlinear inverse problems—both in theory and in practice. In PDE‑based settings, where forward operators are nonlinear and computationally expensive, learned surrogate models are often embedded in classical reconstruction schemes. Augmenting these surrogates with a divergence penalty could encode data knowledge directly into the forward map; and since the divergence can be computed via automatic differentiation independent of the invertibility of the network, this approach would also apply to non‑invertible networks used alongside methods like Tikhonov regularization.
Another exciting opportunity lies in using invertible architectures to enhance interpretability in tasks beyond reconstruction - such as semantic segmentation or classification. By incorporating our regularization terms, one can steer learned mappings toward regions of high data fidelity, yielding more transparent reconstructions for given output samples.
Finally, divergence‑based training offers a principled way to learn invertible proximal operators in Plug‑and‑Play algorithms. Whereas existing approaches often train these operators to approximate the posterior mean \cite{meinhardt}, our loss aligns them with the correct MAP estimator, potentially improving both convergence and robustness when using data‑driven priors.

\section*{Acknowledgements}
Nick Heilenkötter acknowledges the support of the Deutsche Forschungsgemeinschaft (DFG, German Research Foundation) - Project number 281474342/GRK2224/2.


\end{document}

%% file: first_page.pdf_tex
\begingroup%
  \makeatletter%
  \providecommand\color[2][]{%
    \errmessage{(Inkscape) Color is used for the text in Inkscape, but the package 'color.sty' is not loaded}%
    \renewcommand\color[2][]{}%
  }%
  \providecommand\transparent[1]{%
    \errmessage{(Inkscape) Transparency is used (non-zero) for the text in Inkscape, but the package 'transparent.sty' is not loaded}%
    \renewcommand\transparent[1]{}%
  }%
  \providecommand\rotatebox[2]{#2}%
  \newcommand*\fsize{\dimexpr\f@size pt\relax}%
  \newcommand*\lineheight[1]{\fontsize{\fsize}{#1\fsize}\selectfont}%
  \ifx\svgwidth\undefined%
    \setlength{\unitlength}{1544.88188976bp}%
    \ifx\svgscale\undefined%
      \relax%
    \else%
      \setlength{\unitlength}{\unitlength * \real{\svgscale}}%
    \fi%
  \else%
    \setlength{\unitlength}{\svgwidth}%
  \fi%
  \global\let\svgwidth\undefined%
  \global\let\svgscale\undefined%
  \makeatother%
  \begin{picture}(1,0.31192661)%
    \lineheight{1}%
    \setlength\tabcolsep{0pt}%
    \put(0,0){\includegraphics[width=\unitlength,page=1]{first_page.pdf}}%
    \put(0.80347212,0.04722947){\color[rgb]{0,0,0}\makebox(0,0)[lt]{\smash{\begin{tabular}[t]{l}\scalebox{.8}{$\frac{1}{2}\norm{\model(\x)-\zdelta}^2$}\end{tabular}}}}%
    \put(0.02477183,0.02975241){\color[rgb]{0,0,0}\makebox(0,0)[lt]{\smash{\begin{tabular}[t]{l}a)\end{tabular}}}}%
    \put(0.3034339,0.02975241){\color[rgb]{0,0,0}\makebox(0,0)[lt]{\smash{\begin{tabular}[t]{l}\scalebox{0.8}{$\frac{1}{2}\norm{\x-\invmodel(\zdelta)}^2$}\end{tabular}}}}%
    \put(0.55976417,0.04722947){\color[rgb]{0,0,0}\makebox(0,0)[lt]{\smash{\begin{tabular}[t]{l}\scalebox{.8}{$\frac{1}{2}\norm{\model(\x)-\zdelta}^2$}\end{tabular}}}}%
    \put(0.54325806,0.0093625){\color[rgb]{0,0,0}\makebox(0,0)[lt]{\smash{\begin{tabular}[t]{l}\scalebox{0.8}{$-\deltachosen^2\log|\det\Diff\model(\x)|$}\end{tabular}}}}%
    \put(0.80929781,0.0093625){\color[rgb]{0,0,0}\makebox(0,0)[lt]{\smash{\begin{tabular}[t]{l}\scalebox{.8}{$-\deltachosen^2\nabla\cdot\model(\x)$}\end{tabular}}}}%
    \put(0.51898436,0.02975241){\color[rgb]{0,0,0}\makebox(0,0)[lt]{\smash{\begin{tabular}[t]{l}c)\end{tabular}}}}%
    \put(0.77240179,0.02975241){\color[rgb]{0,0,0}\makebox(0,0)[lt]{\smash{\begin{tabular}[t]{l}d)\end{tabular}}}}%
    \put(0.05390027,0.02975241){\color[rgb]{0,0,0}\makebox(0,0)[lt]{\smash{\begin{tabular}[t]{l}\scalebox{0.8}{$\frac{1}{2}\norm{\model(\x)-\zdelta}^2$}\end{tabular}}}}%
    \put(0.27430546,0.02975241){\color[rgb]{0,0,0}\makebox(0,0)[lt]{\smash{\begin{tabular}[t]{l}b)\end{tabular}}}}%
  \end{picture}%
\endgroup%

%% file: bayesian_paper.pdf_tex
\begingroup%
  \makeatletter%
  \providecommand\color[2][]{%
    \errmessage{(Inkscape) Color is used for the text in Inkscape, but the package 'color.sty' is not loaded}%
    \renewcommand\color[2][]{}%
  }%
  \providecommand\transparent[1]{%
    \errmessage{(Inkscape) Transparency is used (non-zero) for the text in Inkscape, but the package 'transparent.sty' is not loaded}%
    \renewcommand\transparent[1]{}%
  }%
  \providecommand\rotatebox[2]{#2}%
  \newcommand*\fsize{\dimexpr\f@size pt\relax}%
  \newcommand*\lineheight[1]{\fontsize{\fsize}{#1\fsize}\selectfont}%
  \ifx\svgwidth\undefined%
    \setlength{\unitlength}{1114.01574803bp}%
    \ifx\svgscale\undefined%
      \relax%
    \else%
      \setlength{\unitlength}{\unitlength * \real{\svgscale}}%
    \fi%
  \else%
    \setlength{\unitlength}{\svgwidth}%
  \fi%
  \global\let\svgwidth\undefined%
  \global\let\svgscale\undefined%
  \makeatother%
  \begin{picture}(1,0.33587786)%
    \lineheight{1}%
    \setlength\tabcolsep{0pt}%
    \put(0,0){\includegraphics[width=\unitlength,page=1]{bayesian_paper.pdf}}%
    \put(0.0756042,0.30249401){\color[rgb]{0,0,0}\makebox(0,0)[lt]{\smash{\begin{tabular}[t]{l}$\X$\end{tabular}}}}%
    \put(0.11798977,0.0100277){\color[rgb]{0,0,0}\makebox(0,0)[lt]{\smash{\begin{tabular}[t]{l}(a) Sampling\end{tabular}}}}%
    \put(0.28021716,0.20511077){\color[rgb]{0,0,0}\makebox(0,0)[lt]{\smash{\begin{tabular}[t]{l}$+\eta$\end{tabular}}}}%
    \put(0.31577309,0.16663689){\color[rgb]{0,0,0}\makebox(0,0)[lt]{\smash{\begin{tabular}[t]{l}$\ydelta$\end{tabular}}}}%
    \put(0.1773324,0.21044737){\color[rgb]{0,0,0}\makebox(0,0)[lt]{\smash{\begin{tabular}[t]{l}$\A$\end{tabular}}}}%
    \put(0.15972572,0.08012559){\color[rgb]{0.29803922,0.44705882,0.69019608}\makebox(0,0)[lt]{\smash{\begin{tabular}[t]{l}$\prior$\end{tabular}}}}%
    \put(0.33185498,0.14704264){\color[rgb]{0.33333333,0.65882353,0.40784314}\makebox(0,0)[lt]{\smash{\begin{tabular}[t]{l}$\pnoise(\A\x-\cdot)$\end{tabular}}}}%
    \put(0.62259024,0.14159201){\color[rgb]{0.33333333,0.65882353,0.40784314}\makebox(0,0)[lt]{\smash{\begin{tabular}[t]{l}$\pnoise(\cdot-\ydelta)$\end{tabular}}}}%
    \put(0.66801022,0.25015202){\color[rgb]{0.86666667,0.51764706,0.32156863}\makebox(0,0)[lt]{\smash{\begin{tabular}[t]{l}$p(\cdot|\ydelta)$\end{tabular}}}}%
    \put(0.59022968,0.0100277){\color[rgb]{0,0,0}\makebox(0,0)[lt]{\smash{\begin{tabular}[t]{l}(b) Reconstruction\end{tabular}}}}%
    \put(0.29508044,0.30249401){\color[rgb]{0,0,0}\makebox(0,0)[lt]{\smash{\begin{tabular}[t]{l}$\Y$\end{tabular}}}}%
    \put(0.78789243,0.30249401){\color[rgb]{0,0,0}\makebox(0,0)[lt]{\smash{\begin{tabular}[t]{l}$\X$\end{tabular}}}}%
    \put(0.60477112,0.30249401){\color[rgb]{0,0,0}\makebox(0,0)[lt]{\smash{\begin{tabular}[t]{l}$\Y$\end{tabular}}}}%
    \put(0.60795993,0.16529041){\color[rgb]{0,0,0}\makebox(0,0)[lt]{\smash{\begin{tabular}[t]{l}$\ydelta$\end{tabular}}}}%
    \put(0.27807165,0.16259744){\color[rgb]{0,0,0}\makebox(0,0)[lt]{\smash{\begin{tabular}[t]{l}$\A\x$\end{tabular}}}}%
    \put(0.80454555,0.15586504){\color[rgb]{0,0,0}\makebox(0,0)[lt]{\smash{\begin{tabular}[t]{l}$\xmap$\end{tabular}}}}%
    \put(0.77223,0.21107074){\color[rgb]{0,0,0}\makebox(0,0)[lt]{\smash{\begin{tabular}[t]{l}$\xpm$\end{tabular}}}}%
    \put(0.87976758,0.09446522){\color[rgb]{0.33333333,0.65882353,0.40784314}\makebox(0,0)[lt]{\smash{\begin{tabular}[t]{l}$\pnoise(\A\cdot-\ydelta)$\end{tabular}}}}%
    \put(0.61241313,0.08509058){\color[rgb]{0.50588235,0.44705882,0.70196078}\makebox(0,0)[lt]{\smash{\begin{tabular}[t]{l}$\py$\end{tabular}}}}%
    \put(0.29806381,0.08368842){\color[rgb]{0.29803922,0.44313725,0.69019608}\transparent{0.74901998}\makebox(0,0)[lt]{\smash{\begin{tabular}[t]{l}$\A_{\#\prior}$\end{tabular}}}}%
  \end{picture}%
\endgroup%

%% file: map_out.pdf_tex
\begingroup%
  \makeatletter%
  \providecommand\color[2][]{%
    \errmessage{(Inkscape) Color is used for the text in Inkscape, but the package 'color.sty' is not loaded}%
    \renewcommand\color[2][]{}%
  }%
  \providecommand\transparent[1]{%
    \errmessage{(Inkscape) Transparency is used (non-zero) for the text in Inkscape, but the package 'transparent.sty' is not loaded}%
    \renewcommand\transparent[1]{}%
  }%
  \providecommand\rotatebox[2]{#2}%
  \newcommand*\fsize{\dimexpr\f@size pt\relax}%
  \newcommand*\lineheight[1]{\fontsize{\fsize}{#1\fsize}\selectfont}%
  \ifx\svgwidth\undefined%
    \setlength{\unitlength}{380bp}%
    \ifx\svgscale\undefined%
      \relax%
    \else%
      \setlength{\unitlength}{\unitlength * \real{\svgscale}}%
    \fi%
  \else%
    \setlength{\unitlength}{\svgwidth}%
  \fi%
  \global\let\svgwidth\undefined%
  \global\let\svgscale\undefined%
  \makeatother%
  \begin{picture}(1,1)%
    \lineheight{1}%
    \setlength\tabcolsep{0pt}%
    \put(0.53330191,0.05052546){\color[rgb]{0.29803922,0.44705882,0.69019608}\makebox(0,0)[lt]{\smash{\begin{tabular}[t]{l}$\delta^2\nabla_\x\log\prior(\xmap)$\end{tabular}}}}%
    \put(0,0){\includegraphics[width=\unitlength,page=1]{map_out.pdf}}%
    \put(0.71268405,0.51627726){\color[rgb]{0,0,0}\makebox(0,0)[lt]{\smash{\begin{tabular}[t]{l}$\ydelta$\end{tabular}}}}%
    \put(0.54788395,0.43660582){\color[rgb]{0,0,0}\makebox(0,0)[lt]{\smash{\begin{tabular}[t]{l}$\xmap$\end{tabular}}}}%
    \put(0,0){\includegraphics[width=\unitlength,page=2]{map_out.pdf}}%
  \end{picture}%
\endgroup%

%% file: post_mean_out.pdf_tex
\begingroup%
  \makeatletter%
  \providecommand\color[2][]{%
    \errmessage{(Inkscape) Color is used for the text in Inkscape, but the package 'color.sty' is not loaded}%
    \renewcommand\color[2][]{}%
  }%
  \providecommand\transparent[1]{%
    \errmessage{(Inkscape) Transparency is used (non-zero) for the text in Inkscape, but the package 'transparent.sty' is not loaded}%
    \renewcommand\transparent[1]{}%
  }%
  \providecommand\rotatebox[2]{#2}%
  \newcommand*\fsize{\dimexpr\f@size pt\relax}%
  \newcommand*\lineheight[1]{\fontsize{\fsize}{#1\fsize}\selectfont}%
  \ifx\svgwidth\undefined%
    \setlength{\unitlength}{380bp}%
    \ifx\svgscale\undefined%
      \relax%
    \else%
      \setlength{\unitlength}{\unitlength * \real{\svgscale}}%
    \fi%
  \else%
    \setlength{\unitlength}{\svgwidth}%
  \fi%
  \global\let\svgwidth\undefined%
  \global\let\svgscale\undefined%
  \makeatother%
  \begin{picture}(1,1)%
    \lineheight{1}%
    \setlength\tabcolsep{0pt}%
    \put(0,0){\includegraphics[width=\unitlength,page=1]{post_mean_out.pdf}}%
    \put(0.7134295,0.52095349){\color[rgb]{0,0,0}\makebox(0,0)[lt]{\smash{\begin{tabular}[t]{l}$\ydelta$\end{tabular}}}}%
    \put(0.53330195,0.05052546){\color[rgb]{0.50196078,0.43921569,0.70588235}\makebox(0,0)[lt]{\smash{\begin{tabular}[t]{l}$\delta^2\nabla_\y\log\py(\ydelta)$\end{tabular}}}}%
    \put(0.54120339,0.45472886){\color[rgb]{0,0,0}\makebox(0,0)[lt]{\smash{\begin{tabular}[t]{l}$\xpm$\end{tabular}}}}%
    \put(0,0){\includegraphics[width=\unitlength,page=2]{post_mean_out.pdf}}%
  \end{picture}%
\endgroup%